\documentclass[11pt]{article}
\usepackage{hyperref}
\usepackage{url}
\usepackage{graphicx,epstopdf}
\usepackage{amssymb,amsfonts,amsmath,amsthm,bm}
\usepackage{algorithm2e,framed,algorithmic}

\def\math#1{$#1$}

\def\v#1{{\mathbf #1}}
\def\frac#1#2{{#1\over #2}}

\def\x{{\mathbf x}}
\def\y{{\mathbf y}}

\def\a{{\mathbf a}}
\def\b{{\mathbf b}}

\def\norm#1{{\|#1\|}}

\def\r#1{{(\ref{#1})}}

\def\dotfil{\leaders\hbox to 1.5mm{.}\hfill}
\newcounter{rmnum}
\def\RN#1{\setcounter{rmnum}{#1}\uppercase\expandafter{\romannumeral\value{rmnum}}}
\def\rn#1{\setcounter{rmnum}{#1}\expandafter{\romannumeral\value{rmnum}}}

\newcommand{\TNorm }[1]{\mbox{}\left\|#1\right\|_2  }
\newcommand{\TNormS}[1]{\mbox{}\left\|#1\right\|_2^2}

\newcommand{\setlinespacing}[1]%
           {\setlength{\baselineskip}{#1 \defbaselineskip}}

\newcommand{\rank}[1]{{\bf rank}{\left(#1\right)}}
\newcommand{\abs }[1]{\left|#1\right|}

\newtheorem{lemma}{Lemma}
\newtheorem{theorem}{Theorem}

\newcommand{\mat}[1]{{\ensuremath{\bm{\mathrm{#1}}}}}

\def\rank{\hbox{\rm rank}}

\def\a{{\bm \alpha}}
\def\w{{\mathbf w}}

\def\ta{\tilde{\bm \alpha}}
\def\tw{\tilde{\mathbf w}}
\def\b{{\mathbf b}}
\def\e{{\mathbf e}}

\def\x{{\mathbf x}}

\def\v{{\mathbf v}}

\def\matA{\mat{A}}
\def\matB{\mat{B}}

\def\matD{\mat{D}}
\def\matE{\mat{E}}
\def\matG{\mat{G}}

\def\matI{\mat{I}}

\def\matP{\mat{P}}
\def\matQ{\mat{Q}}
\def\matR{\mat{R}}

\def\matS{\mat{S}}

\def\matU{\mat{U}}
\def\matV{\mat{V}}

\def\matX{\mat{X}}
\def\matY{\mat{Y}}

\def\matSig{\mat{\Sigma}}

\title{Feature Selection for Linear SVM with Provable Guarantees}

\author{
Saurabh Paul  
\thanks{Computer Science Department, 
Rensselaer Polytechnic Institute,
Troy, NY, USA, 
\texttt{pauls2@rpi.edu}
} 
\and Malik Magdon-Ismail 
\thanks {Computer Science Department, 
Rensselaer Polytechnic Institute, 
Troy, NY, USA
\texttt{magdon@cs.rpi.edu} 
}
\and Petros Drineas 
\thanks {Computer Science Department, 
Rensselaer Polytechnic Institute, 
Troy, NY, USA, 
\texttt{drinep@cs.rpi.edu} 
}
}

\begin{document}

\maketitle

\begin{abstract}
We give two
\emph{provably} accurate  
feature-selection techniques for the linear SVM. The algorithms run in deterministic and randomized time respectively.
Our algorithms can be used in an unsupervised or supervised setting. 
The supervised approach is based on sampling features from 
support vectors.
We prove that the margin in the feature space is preserved 
to within $\epsilon$-relative error of the margin in the full 
feature space in the worst-case.
In the unsupervised setting, 
we also provide worst-case guarantees of the radius of the minimum enclosing ball, thereby
ensuring comparable generalization as in the full feature space and 
resolving an open problem posed in \cite{Dasgup07}.
We present extensive experiments on 
real-world datasets to support our theory 
and to demonstrate that our method is competitive and often 
better than prior state-of-the-art, for which there are no known provable 
guarantees.
\end{abstract}

\section{Introduction}
The linear Support Vector Machine (SVM) 
is a popular classification method \cite{Chris00}.
Few theoretical
results exist for feature selection with SVMs.
Empirically, numerous feature selection techniques work well
(e.g. \cite{Guyon02,LPSVM}).
We present a \emph{deterministic} and a \emph{randomized}
feature selection technique for the
linear SVM  with a \textit{provable worst-case performance guarantee} on the
margin.
The feature selection is unsupervised if features are selected 
obliviously to the data labels; otherwise, it is supervised.
Our algorithms can be used in an unsupervised or supervised setting.
In the unsupervised setting, our algorithm selects a number of features
proportional to the rank of the data and preserves both
 the margin and radius of minimum enclosing ball 
to within $\epsilon$-relative error in the worst-case, thus resolving an open problem 
posed in \cite{Dasgup07}. 
In the supervised setting,
our algorithm selects \math{O(\# \text{support vectors})} features
using 
only the set of support vectors, and 
preserves the margin for the support vectors
  to within $\epsilon$-relative error in the worst-case.\\

\noindent
{\bf SVM basics.}
The training
data has 
 $n$ points $\x_i \in \mathbb{R}^d$,
with respective labels $y_i \in \{-1,+1\}$ for $i=1\ldots n$.
For linearly separable data, the primal SVM learning problem
constructs a hyperplane
\math{\w^*} which maximizes the geometric
\emph{margin}
(the minimum distance of a data point to the
hyperplane), while separating the data.
For non-separable data the ``soft'' 1-norm margin is maximized.
The dual lagrangian formulation of the soft 1-norm SVM 
reduces to the following quadratic program:
\begin{equation}
\begin{array}{rl}
\max\limits_{\alpha_i }:&\sum\limits_{i=1}^n \alpha_i  - \frac{1}{2} \sum\limits_{i,j=1}^n \alpha_i \alpha_j y_i y_j \x_i^T\x_j \\[10pt]
\text{subject to:}&\sum\limits_{i=1}^{n}y_i \alpha_i = 0;\quad
0 \leq  \alpha_i \leq C, \quad i=1\ldots n.
\end{array}
\label{eqn:svm1}
\end{equation}
 The regularizer $C$ is part of the input and the hyperplane classifier 
can be constructed from the~\math{\alpha_i}.
The out-of-sample performance is related to the
\math{VC}-dimension of the resulting ``\emph{fat}''-separator. 
Assuming that the data lie in a
ball of radius \math{B}, and that the hypothesis set consists of hyperplanes of width \math{\gamma} (the margin), then the
\math{VC}-dimension is \math{O(B^2/\gamma^2)} (\cite{Vapnik71}).
Thus, by the \math{VC}-bound (\cite{Vapnik98}),
the out-of-sample error is bounded by the in-sample error and a term
monotonic in \math{B^2/\gamma^2}.\\

\noindent
{\bf Our Contributions.} 
We give two 
provably accurate feature selection techniques 
for linear SVM in both unsupervised and supervised settings 
with worst-case performance guarantees on the margin.
We use the single set spectral sparsification technique from \cite{BSS09}
as our deterministic algorithm (the algorithm runs in deterministic time, hence the name `deterministic') and leverage-score sampling (\cite{Dasgup07}) as the randomized algorithm.
We give a new simple method of extending unsupervised feature selection 
to supervised in the context of SVMs by running the unsupervised
technique on the support vectors.
This allows us to select only \math{O(\# \text{support vectors})}
features for the deterministic algorithm (\math{\tilde{O}(\# \text{support vectors})} features for the randomized algorithm, where $\tilde{O}$ hides the logarithmic factors) while still preserving the margin on the support vectors. Since the
support vectors are a sufficient statistic for learning a linear SVM, preserving
the margin on the support vectors should be enough for learning on all the data
with the sampled feature set.

More formally, let \math{\gamma^*} be the optimal margin for the support
vector set (which is also the optimal margin for all the data). The optimal
margin \math{\gamma^*} is obtained  by solving the SVM optimization problem 
using all the features. For a suitably chosen number of features $r$, let 
\math{\tilde\gamma^*} be the optimal margin obtained by 
solving the SVM optimization problem using the support vectors 
in the sampled feature space. We prove that the margin is 
preserved to within \math{\epsilon}-relative error:
$\tilde\gamma^{*2}\ge \left(1-\epsilon \right)\gamma^{*2}.$
For the deterministic algorithm, the number of features required is \math{r=O(\# \text{support vectors}/\epsilon^2)}, whereas the randomized algorithm requires \math{r=\tilde{O}(\# \text{support vectors}/\epsilon^2)} features to be selected. 

In the unsupervised setting, by running our algorithm on all the data, instead
of only the support vectors, we get a stronger result statistically, but
using more features. The deterministic algorithm requires $O(\rho/\epsilon^2)$ features to be selected, while the randomized algorithm requires $O(\rho/\epsilon^2 \log(\rho/\epsilon^2))$ features to be selected. 
Again, defining 
\math{\tilde\gamma^*} as the optimal margin obtained by 
solving the SVM optimization problem using all the data  
in the sampled feature space, we prove that 
\math{\tilde\gamma^{*2}\ge \left(1-\epsilon \right)\gamma^{*2}.}
We can now also prove that the data radius is preserved,
$
\tilde B^2 \le \left(1+\epsilon \right) B^2.$
This means that \math{B^2/\gamma^{*2}} is preserved to within 
\math{\epsilon}-relative error, which means that the generalisation error is 
also preserved to within \math{\epsilon}-relative error. 
The rank of the data is the effective dimension of the data, and one can construct this many combinations of pure features that preserve the geometry of the SVM exactly. What makes our result non-trivial is that we
select this many pure features and preserve the geometry of the
SVM to within \math{\epsilon}-relative error.

On the practical side, we provide an efficient
 heuristic for our supervised feature selection using BSS
 which allows our 
algorithm to scale-up to large datasets. 
While the main focus of this paper is theoretical, we compare both supervised and unsupervised versions of feature selection using single-set spectral sparsification and leverage-score sampling with the corresponding supervised and unsupervised forms of Recursive Feature Elimination (RFE) (\cite{Guyon02}), LPSVM (\cite{LPSVM}), uniform sampling and rank-revealing QR factorization (RRQR) based method of column selection. Feature selection based on the single-set spectral sparsification and leverage-score sampling technique is competitive and often better than RFE and LPSVM, and none of the prior art comes 
with provable performance guarantees in either the supervised or
 unsupervised setting.\\

\noindent\textbf{Related Work.}
All the prior art is heuristic in that there are no performance 
guarantees; nevertheless, they have been empirically tested against each other.
Our algorithm comes with provable bounds, and performs comparably or better
in empirical tests. We give a short survey of the prior art:
Guyon et al. \cite{Guyon02} and Rakotomamonjy \cite{Rak03} 
proposed SVM based criteria to rank features based on the weights.
Weston et al. \cite{Weston00} formulated a combinatorial optimization problem to select features by minimising \math{B^2/\gamma^2}. Weston et al. \cite{Weston03} used the zero norm to perform error minimization and feature selection in one step. A Newton based method of feature selection using
 linear programming was given in \cite{LPSVM}. Tan et al. \cite{Tan10} formulated the $\ell_0$-norm Sparse SVM using mixed integer programming. Do et al. \cite{Do09} proposed R-SVM which performs feature selection and ranking by optimizing the radius-margin bound with a scaling factor, and extend this work in
\cite{Do13} using a quadratic 
optimization problem with quadratic constraints. 
Another line of work includes the doubly regularised Support Vector Machine (DrSVM) \cite{Wang06} which uses a mixture of \math{\ell_2}-norm and \math{\ell_1}-norm penalties to solve the SVM optimization problem and perform variable selection. Subsequent works on DrSVM involve reducing the computational bottleneck (\cite{Wang08},\cite{Ye11}).
Gilad-Bachrach et al. \cite{Gilad04} formulate the margin as a function of set of features and score to sets of features according to the margin induced.
Park et al. \cite{Park} studied the Fisher consistency and oracle property of penalized SVM where the dimension of inputs is fixed and showed that their method is able to identify the right model in most cases.\\
Paul et al. \cite{Paulaistats, Paultkdd} used random projections on linear SVM and showed that the margin and data-radius are preserved. However, this is different from our work, since they used linear combinations of features and we select pure features.\\
BSS and leverage-score sampling have been used to select features for \math{k}-means 
(\cite{Bouts13,Bouts09}), regularized least-squares classifier (\cite{Dasgup07, Paul14}). Our work further expands research into 
sparsification algorithms for machine learning.

\section{Background}
\label{subsec:notn}
\noindent\textbf{Notation:} $\matA, \matB, \ldots$ denote matrices and $\a, \b, \ldots$ denote column vectors; $\e_i$ (for all $i=1\ldots n$) is the standard basis, whose dimensionality will be clear from context; and $\matI_n$ is the  $n \times n$ identity matrix. The Singular Value Decomposition (SVD) of a matrix $\matA \in \mathbb{R}^{n \times d}$ of rank $\rho \leq \min\left\{n,d\right\}$ is equal to %
$ \matA = \matU \matSig \matV^T,$
where $\matU \in \mathbb{R}^{n \times \rho}$ is an orthogonal matrix containing the left singular vectors, $\matSig \in \mathbb{R}^{\rho \times \rho}$ is a diagonal matrix containing the singular values $\sigma_1 \geq \sigma_2  \geq \ldots \sigma_{\rho} > 0$, and $\matV \in \mathbb{R}^{d \times \rho}$ is a matrix containing the right singular vectors.
The spectral norm of $\matA$ is  $\TNorm{\matA} = \sigma_1$. 

\noindent\textbf{Matrix Sampling Formalism:}  Let $\matA$ be the data matrix consisting of $n$ points and $d$ dimensions, $\matS \in \mathbb{R}^{d\times r}$ be a matrix such that $\matA\matS \in \mathbb{R}^{n\times r}$ contains $r$ columns of $\matA$ ($\matS$ is a sampling matrix as it samples $r$ columns of $\matA).$
 Let $\matD \in \mathbb{R}^{r\times r}$ be the diagonal matrix such that $\matA\matS\matD \in \mathbb{R}^{n\times r}$ rescales the columns of $\matA$ that are in $\matA\matS.$ We will replace the sampling and re-scaling matrices by a single matrix $\matR \in \mathbb{R}^{d\times r}$, where $\matR = \matS\matD$ 
first samples and then rescales $r$ columns of $\matA$.

Let $\matX$ be a generic data matrix in \math{d} dimensions
whose rows are data vectors $\x_i^T$, and let $\matY$ be the  
diagonal label matrix whose diagonal entries are the labels, 
$\matY_{ii}=y_{i}$.
Let $\a = \left[\alpha_1,\alpha_2,\ldots,\alpha_n\right] \in \mathbb{R}^n$ be
the vector of lagrange multipliers to be determined by solving eqn.~\r{eqn:svm1A}. In matrix form, the SVM dual optimization
problem is
\begin{equation}
\begin{array}{rl}
 \max_{\a }:& \bm{1}^T\a-\frac12 \a^T \matY \matX \matX^T \matY\a \\[5pt]
\text{subject to:}&  \bm 1^T \matY \a = 0;  \qquad
\bm{0} \le\a \le \bm{C}.
\end{array}
\label{eqn:svm1A}
\end{equation}
\noindent (In the above, \math{\bm1,\ \bm0,\ \bm{C}} are  vectors with the implied constant entry.) 
When the data and label
 matrices contain all the data, we will emphasize this using the
notation 
\math{\matX^{\textbf{tr}}\in\mathbb{R}^{n\times d},\ \matY^{\textbf{tr}}
\in\mathbb{R}^{n\times n}}.
Solving~\r{eqn:svm1A} with these full data matrices gives a 
solution \math{\dot{\bm\alpha}^*}.
The data $\x_i$ for which \math{\dot{\alpha}_i^*>0} are support 
vectors and we denote by \math{\matX^{\textbf{sv}}
\in\mathbb{R}^{p\times d},\ \matY^{\textbf{sv}}\in\mathbb{R}^{p\times p}}
the data and label
 matrices containing only the \math{p} support vectors.
Solving \r{eqn:svm1A} with \math{(\matX^{\textbf{tr}},\matY^{\textbf{tr}})}
or \math{(\matX^{\textbf{sv}},\matY^{\textbf{sv}})} result in the same
classifier. Let 
 $\a^{*}$ be the solution to \r{eqn:svm1A} for the 
support vector data.
The optimal separating hyperplane is 
 $\w^* =  (\matX^{\textbf{tr}})^T\matY^{\textbf{tr}}\dot{\a}^{*}= (\matX^{\textbf{sv}})^T\matY^{\textbf{sv}} \a^{*}$, where  $\matX^{\textbf{sv}}$ is the support vector matrix.
The geometric margin is \math{\gamma^*=1/\TNorm{\w^*}}, where
$\TNormS{\w^*} = \sum_{i=1}^n \alpha_i^*$.
The data radius is \math{B=\min_{\x^*}\max_{\x_i} \TNorm{\x_i-\x^*}}.

Our goal is to study how the SVM performs when run in the sampled feature
space. 
Let \math{\matX,\ \matY} be data and label matrices 
(such as those above)
and 
\math{\matR\in\mathbb{R}^{d\times r}} a sampling and rescaling matrix
which selects \math{r} columns of~\math{\matX}.
The transformed dataset into the \math{r} selected
features is \math{\tilde\matX=\matX\matR}, 
and the SVM optimization problem in this feature space becomes
\begin{equation}
\begin{array}{rl}
\max_{\hat\a }:&\bm{1}^T\hat\a-\frac12\hat\a^T\matY\matX\matR\matR^T
\matX^T \matY\hat\a, \\[5pt]
\text{subject to:}&  \bm1^T\matY\hat\a=0; 
\qquad \bm0\le\hat\a\le \bm{C}.
\end{array}
\label{eqn:svm3}
\end{equation}
For the supervised setting, we select features from the support 
vector matrix and so we set $\matX=\matX^{\textbf{sv}}$ and
\math{\matY=\matY^{\textbf{sv}}} and we select \math{r_1\ll d} features using 
\math{\matR}.
For the unsupervised setting, we select features from the full data 
matrix and so we set $\matX=\matX^{\textbf{tr}}$ and
\math{\matY=\matY^{\textbf{tr}}} and we select \math{r_2\ll d} features using 
\math{\matR}.

\section{Our main tools}
\label{sec:main_tool}
In this section, we describe our main tools of feature selection from the numerical linear algebra literature, namely single-set spectral sparsification and leverage-score sampling.\\

\noindent
\textbf{Single-set Spectral Sparsification.} The BSS algorithm (\cite{BSS09}) is  a deterministic
greedy technique that selects columns one at a time. The algorithm samples $r$ columns in deterministic time, hence the name deterministic sampling. Consider the input matrix as a set of $d$ column vectors $\matV^T = \left[ \v_1, \v_2,....,\v_d \right]$, with $\v_i \in \mathbb{R}^\ell \left(i = 1,..,d\right).$ Given $\ell$ and $r>\ell$, we iterate over $\tau = 0,1,2,.. r-1$. Define the parameters $L_{\tau} = \tau - \sqrt{r\ell}, \delta_L = 1, \delta_U = \left(1+\sqrt{\ell/r}\right)/\left(1 - \sqrt{\ell/r}\right)$ and $U_{\tau} = \delta_U\left(\tau + \sqrt{\ell r}\right)$. For $U, L \in \mathbb{R}$ and $\matA \in \mathbb{R}^{\ell\times \ell}$ a symmetric positive definite matrix with eigenvalues $\lambda_1, \lambda_2,...,\lambda_\ell$, define 
$\Phi\left(L,\matA\right) = \sum_{i=1}^\ell \frac{1}{\lambda_i-L}$ and 
$\hat{\Phi}\left(U,\matA\right) =  \sum_{i=1}^\ell \frac{1}{U-\lambda_i}$
as the lower and upper potentials respectively. These potential functions measure how far the eigenvalues of $\matA$ are from the upper and lower barriers $U$ and $L$ respectively. We define $\mathcal{L}\left(\v, \delta_L, \matA, L\right)$  and $\mathcal{U}\left(\v, \delta_U, \matA, U\right)$  as follows:

$$ \mathcal{L}\left(\v, \delta_L, \matA, L\right)=
\frac{\v^T \left(\matA - \left(L+\delta_L\right)\matI_\ell\right)^{-2}\v}{\Phi\left(L+\delta_L,\matA\right) - \Phi\left(L,\matA\right)} - \v^T\left(\matA - \left(L+\delta_L\right)\matI_\ell\right)^{-1}\v$$
$$\mathcal{U}\left(\v, \delta_U, \matA, U\right)=
\frac{\v^T \left(\left(U+\delta_U\right)\matI_\ell - \matA\right)^{-2}\v}{ \hat{\Phi}\left(U,\matA\right)-\hat{\Phi}\left(U+\delta_U,\matA\right)} + \v^T\left(\left(U+\delta_U\right)\matI_\ell - \matA\right)^{-1}\v.$$ 
At every iteration, there exists an index $i_{\tau}$ and a weight $t_{\tau}>0$ such that, ${t_{\tau}}^{-1}\leq\mathcal{L}\left(\v_i, \delta_L, \matA, L\right)$ and 
${t_{\tau}}^{-1}\geq \mathcal{U}\left(\v_i, \delta_U, \matA, U\right).$ Thus, there will be at most $r$ columns selected after $\tau$ iterations. The running time of the algorithm is dominated by the search for an index $i_\tau$ satisfying $ \mathcal{U}\left(\v_i,\delta_U,\matA_{\tau},U_{\tau} \right) \leq \mathcal{L}\left(\v_i,\delta_L, \matA_{\tau},L_{\tau} \right)$ and computing the weight $t_{\tau}.$ One needs to compute the upper and lower potentials $\hat{\Phi}\left(U,\matA\right)$ and $\Phi\left(L,\matA\right)$ and hence the eigenvalues of $\matA$. Cost per iteration is $O\left(\ell^3\right)$ and the total cost is $O\left(r\ell^3\right).$ For $i=1,..,d$, we need to compute $\mathcal{L}$ and $\mathcal{U}$ for every $\v_i$ which can be done in $O\left(d\ell^2 \right)$ for every iteration, for a total
of $O\left(rd\ell^2 \right).$ Thus total running time of the algorithm is  $O\left(rd\ell^2 \right).$ We include the algorithm as Algorithm~\ref{alg:alg_ssp}.

\begin{algorithm}[!htb]
\begin{framed}
\textbf{Input:} $\matV^T = [ \v_1, \v_2, ... \v_d ] \in \mathbb{R}^{\ell \times d}$ with $\v_i \in \mathbb{R}^{\ell}$ and $r>\ell$. \\
\textbf{Output:} Matrices $\matS \in \mathbb{R}^{d\times r}, \matD \in \mathbb{R}^{r\times r}$.\\
%
1. Initialize $\matA_0 = \mathbf{0}_{\ell \times \ell}$, $\matS =\mathbf{0}_{d\times r}, \matD =\mathbf{0}_{r\times r}$.\\
2. Set constants $\delta_L = 1$ and $\delta_U = \left(1+\sqrt{\ell/r}\right)/\left(1-\sqrt{\ell/r}\right)$. \\
3. \textbf{for} $\tau = 0$ to $r-1$ \textbf{do}
\begin{itemize}
	\item Let $L_{\tau} = \tau - \sqrt{r\ell} ; U_{\tau} = \delta_U \left(\tau+\sqrt{\ell r}\right)$.
	\item Pick index $i \in \{1,2,..d \}$ and number $t_{\tau}>0$, such that
	$$ \mathcal{U}\left(\v_i,\delta_U,\matA_{\tau},U_{\tau} \right) \leq \mathcal{L}\left(\v_i,\delta_L, \matA_{\tau},L_{\tau} \right). $$
	\item Let $ t_{\tau}^{-1} = \frac{1}{2} \left( \mathcal{U}\left(\v_i,\delta_U,\matA_{\tau},U_{\tau} \right)+ \mathcal{L}\left(\v_i,\delta_L,\matA_{\tau},L_{\tau} \right) \right)$
	\item Update $\matA_{\tau+1} = \matA_{\tau} + t_{\tau}\v_i\v_i^T$ ; set $\matS_{i_\tau,\tau+1}=1$ and $\matD_{\tau+1,\tau+1}=1/\sqrt{t_{\tau}}$.
\end{itemize}
4. \textbf{end for} \\
5. Multiply all the weights in $\matD$ by $\sqrt{r^{-1}\left(1-\sqrt{\left(\ell/r \right)}\right)}.$ \\
6. Return $\matS$ and $\matD.$

\end{framed}
\caption{Single-set Spectral Sparsification}
\label{alg:alg_ssp}
\end{algorithm}

We present the following lemma for the single-set spectral sparsification algorithm.

\begin{lemma}\label{lem:bss}
\textbf{BSS} (\cite{BSS09}): Given $\matV \in \mathbb{R}^{d \times \ell}$ satisfying $\matV^T\matV = \matI_\ell$ and $r>\ell$, we can deterministically construct sampling and rescaling matrices $\matS\in\mathbb{R}^{d\times r}$ and $\matD\in\mathbb{R}^{r \times r}$ with $\matR=\matS\matD$, such that, for all $\y \in \mathbb{R}^\ell :$
$ \left(1-\sqrt{\ell/r}\right)^2 \TNormS{\matV\y} \le \TNormS{\matV^T\matR\y} \le  \left( 1+ \sqrt{\ell/r} \right)^2 \TNormS{\matV\y}.$
\end{lemma}

We now present a slightly modified version of Lemma~\ref{lem:bss} for our theorems.

\begin{lemma} \label{lem:ssp-app}
Given $\matV \in \mathbb{R}^{d \times \ell}$ satisfying $\matV^T\matV = \matI_\ell$ and $r>\ell$, we can deterministically construct sampling and rescaling matrices $\matS \in \mathbb{R}^{d \times r}$ and $\matD \in \mathbb{R}^{r \times r}$ such that for $\matR= \matS\matD$, $\TNorm{ \matV^T\matV - \matV^T\matR\matR^T\matV} \leq 3\sqrt{\ell/r}$
\end{lemma}

\begin{proof}
From Lemma~\ref{lem:bss}, it follows,
$\sigma_\ell\left(\matV^T\matR\matR^T\matV\right)\ge\left(1-\sqrt{\ell/r}\right)^2,$  
$\sigma_1 \left(\matV^T\matR\matR^T\matV\right)\le\left(1+\sqrt{\ell/r}\right)^2.$ 
Thus,
$\lambda_{max}\left(\matV^T\matV-\matV^T\matR\matR^T\matV\right)\le\left(1-\left(1-\sqrt{\ell/r}\right)^2\right)\le 2\sqrt{\ell/r}.$
Similarly,
$\lambda_{min} \left(\matV^T\matV - \matV^T\matR\matR^T\matV  \right) \ge \left( 1 - \left(1 +\sqrt{\ell/r}\right)^2 \right)  \ge 3\sqrt{\ell/r}.$
Combining these two results, we have
$\TNorm{\matV^T\matV - \matV^T\matR\matR^T\matV} \leq 3\sqrt{\ell/r}.$
\end{proof}

\noindent
\textbf{Leverage-Score Sampling.}
Our randomized feature selection method is based on importance sampling or the so-called leverage-score sampling of \cite{Dasgup07}. Let $\matV$ be the top-$\rho$ right singular vectors of the training set $\matX$. A carefully chosen probability distribution of the form 
\vskip -0.2cm
\begin{equation}
p_i = \frac{\TNormS{\matV_{i}}}{n}, \text{ for } i=1,2,...,d, 
\label{eqn:eqnlvg}
\end{equation}
i.e. proportional to the squared Euclidean norms of the rows of the right-singular vectors is constructed. Select $r$ rows of $\matV$ in i.i.d trials and re-scale the rows with $1/\sqrt{p_i}$. The time complexity is dominated by the time to compute the SVD of $\matX$.

\begin{lemma} \label{lem:lvgscr}\vskip -0.1cm
Let $\epsilon \in(0,1/2]$ be an accuracy parameter and $\delta \in(0,1)$ be the failure probability. Given $\matV \in \mathbb{R}^{d \times \ell}$ satisfying $\matV^T\matV = \matI_\ell.$ Let $\tilde{p} = min\{1, rp_i\}$, let $p_i$ be as Eqn.~\ref{eqn:eqnlvg} and let $r = O\left(\frac{n}{\epsilon^2} \log\left(\frac{n}{\epsilon^2 \sqrt{\delta}}\right) \right)$. Construct the sampling and rescaling matrix $\matR$.  Then with probability at least 0.99, 
$ \TNorm{\matV^T\matV - \matV^T\matR^T\matR\matV} \leq \epsilon.$
\end{lemma}
\vskip -0.2cm

\section{Theoretical Analysis}
Though our feature selection algorithms are relatively simple, 
we show that running the linear SVM
in the feature space results in a classifier with provably comparable
margin to the SVM classifier obtained from the full feature space.
Our main results are in Theorems \ref{thm:main_technical} \& \ref{thm:genthm}. We state the theorems for BSS, but similar theorems can be stated for leverage-score sampling.

\subsection{Margin is preserved by Supervised Feature Selection}
Theorem \ref{thm:main_technical} says that you get comparable
margin from solving the
SVM on the support vectors (equivalently all the data) and 
from solving the
SVM on support vectors in a feature space with only 
\math{O(\#\text{support vectors})} features.
\begin{theorem}\label{thm:main_technical}
Given \math{\epsilon\in(0,1)}, run supervised BSS-feature selection on
\math{\matX^{\textbf{sv}}} with \math{r_1=36p/\epsilon^2},
to obtain the feature sampling and rescaling matrix \math{\matR}.
Let \math{\gamma^*} and \math{\tilde\gamma^*} be the margins obtained by
solving the SVM dual \r{eqn:svm1A} with 
\math{(\matX^{\textbf{sv}},\matY^{\textbf{sv}})} and 
\math{(\matX^{\textbf{sv}} \matR,\matY^{\textbf{sv}})}
respectively. Then,
$\tilde\gamma^{*2}\ge \left(1-\epsilon\right) \gamma^{*2}.$
\end{theorem}
\begin{proof} 
Let $\matX^{\textbf{tr}} \in \mathbb{R}^{n\times d},$ $\matY^{\textbf{tr}} \in \mathbb{R}^{n \times n}$  be the feature matrix and class labels of the training set (as defined in Section~\ref{subsec:notn}) and let $\dot{\a}^* = \left[\alpha_1^*,\alpha_2^*,\ldots,\alpha_n^*\right]^T \in \mathbb{R}^n$ be the vector achieving the optimal solution for the problem of eqn.~(\ref{eqn:svm1A}). Then,

\begin{equation} \label{eqn:svm_org}
Z_{opt} = \sum_{j=1}^n \dot{\alpha}_j^* - \frac{1}{2} \dot{\a}^{*T} \matY^{\textbf{tr}} \matX^{\textbf{tr}}  \left(\matX^{\textbf{tr}}\right)^T \matY^{\textbf{tr}} \dot{\a}^* 
\end{equation}
Let $p \leq n$ be the support vectors with $\dot{\alpha}_j>0$. Let $\a^* = \left[\alpha_1^*,\alpha_2^*,\ldots,\alpha_p^*\right]^T \in \mathbb{R}^p$ be the vector achieving the optimal solution for the problem of eqn.~(\ref{eqn:svm_org}). Let $\matX^{\textbf{sv}} \in \mathbb{R}^{p\times d}$, $\matY^{\textbf{sv}} \in \mathbb{R}^{p\times p}$ be the support vector matrix and the corresponding labels respectively. Let $\matE=\matV^T\matV - \matV^T \matR\matR^T \matV$. Then, we can write eqn~(\ref{eqn:svm_org}) in terms of support vectors as,
\begin{eqnarray}
Z_{opt} &=& \sum_{i=1}^p \alpha_i^* - \frac{1}{2} \a^{*T} \matY^{\textbf{sv}} \matX^{\textbf{sv}} \left(\matX^{\textbf{sv}}\right)^T \matY^{\textbf{sv}} \a^* \label{eqn:svm_sv} \nonumber \\
 &=& \sum_{i=1}^p \alpha_i^*  - \frac{1}{2} \a^{*T} \matY^{\textbf{sv}} \matU \matSig \matV^T \matV \matSig \matU^T \matY^{\textbf{sv}} \a^* \nonumber \\
\label{eqn: thm_eqn1} &=& \sum_{i=1}^p \alpha_i^* - \frac{1}{2} \a^{*T} \matY^{\textbf{sv}} \matU \matSig \matV^T \matR \matR^T \matV \matSig \matU^T \matY^{\textbf{sv}} \a^* \nonumber \\
&&- \frac{1}{2} \a^{*T} \matY^{\textbf{sv}} \matU \matSig \matE \matSig \matU^T \matY^{\textbf{sv}} \a^* .
\end{eqnarray}
Let $\ta^* = \left[\tilde{\alpha}_1^*,\tilde{\alpha}_2^*,\ldots,\tilde{\alpha}_p^*\right]^T \in\mathbb{R}^p$ be the vector achieving the optimal solution for the dimensionally-reduced SVM problem of eqn.~(\ref{eqn:svm_sv})
using \math{\tilde\matX^{\textbf{sv}}=\matX^{\textbf{sv}}\matR}. Using the SVD of $\matX^{\textbf{sv}}$,
\begin{equation}
\label{eqn:thm_eqn2}
\tilde{Z}_{opt} = \sum_{i=1}^p \tilde{\alpha}_i^* - \frac{1}{2}\ta^{*T} \matY^{\textbf{sv}}\matU \matSig \matV^T \matR \matR^T \matV \matSig \matU^T \matY^{\textbf{sv}} \ta^*.
\end{equation}
Since the constraints on \math{\a^*,\ta^*} do not depend on the data
it is clear that $\ta^*$ is a feasible solution for the problem of eqn.~(\ref{eqn:svm_sv}). Thus, from the optimality of $\a^*$, and using eqn.~(\ref{eqn:thm_eqn2}), it follows that
\begin{eqnarray}
Z_{opt} &=& \sum_{i=1}^p \alpha_i^*  - \frac{1}{2} \a^{*T} \matY^{\textbf{sv}} \matU \matSig \matV^T \matR \matR^T \matV \matSig \matU^T \matY^{\textbf{sv}} \a^* \nonumber \\
&&- \frac{1}{2} \a^{*T} \matY^{\textbf{sv}} \matU \matSig \matE \matSig \matU^T \matY^{\textbf{sv}} \a^* \nonumber \\
&\geq& \sum_{i=1}^p \tilde{\alpha}_i^*  - \frac{1}{2} \ta^{*T} \matY^{\textbf{sv}} \matU \matSig \matV^T \matR \matR^T \matV \matSig \matU^T \matY^{\textbf{sv}} \ta^* \nonumber \\ 
&&- \frac{1}{2} \ta^{*T} \matY^{\textbf{sv}} \matU \matSig \matE \matSig \matU^T \matY^{\textbf{sv}} \ta^* \nonumber \\
\label{eqn: thm_eqn3} &=& \tilde{Z}_{opt} - \frac{1}{2} \ta^{*T} \matY^{\textbf{sv}} \matU \matSig \matE \matSig \matU^T \matY^{\textbf{sv}} \ta^*.
\end{eqnarray}
We now analyze the second term using standard submultiplicativity properties and $\matV^T\matV=\matI$. Taking $\matQ = \ta^{*T}\matY^{\textbf{sv}}\matU\matSig$,
\begin{eqnarray}
\nonumber && \frac{1}{2}\ta^{*T}\matY^{\textbf{sv}}\matU\matSig\matE\matSig\matU^T\matY^{\textbf{sv}}\ta^*  \\
&\leq& \frac{1}{2}\TNorm{\matQ} \TNorm{\matE} \TNorm{\matQ^T} \nonumber \\
\nonumber &=& \frac{1}{2} \TNorm{\matE} \TNormS{\matQ} \\
\nonumber &=& \frac{1}{2} \TNorm{\matE} \TNormS{\ta^{*T}\matY^{\textbf{sv}} \matU\matSig \matV^T} \\
\label{eqn:thm_eqn4} &=& \frac{1}{2} \TNorm{\matE} \TNormS{\ta^{*T}\matY^{\textbf{sv}}\matX^{\textbf{sv}}}.
\end{eqnarray}
%
Combining eqns.~\eqref{eqn: thm_eqn3} and~\eqref{eqn:thm_eqn4}, we get
\begin{eqnarray}
\label{eqn:thm_eqn5}
Z_{opt} &\geq & \tilde{Z}_{opt} - \frac{1}{2} \TNorm{\matE} \TNormS{\ta^{*T}\matY^{\textbf{sv}}\matX^{\textbf{sv}}}.
\end{eqnarray}
We now proceed to bound the second term in the right-hand side of the above equation. Towards that end, we bound the difference:
%
\begin{eqnarray*}
%
&& \abs{\ta^{*T} \matY^{\textbf{sv}}\matU\matSig\left(\matV^T\matR\matR^T\matV-\matV^T\matV\right)\matSig\matU^T\matY^{\textbf{sv}}\ta^*} \nonumber \\
&=& \abs{\ta^{*T} \matY^{\textbf{sv}}\matU\matSig\left(-\matE\right)\matSig\matU^T\matY^{\textbf{sv}}\ta^*} \nonumber \\
&\leq& \TNorm{\matE}\TNormS{\ta^{*T} \matY^{\textbf{sv}}\matU\matSig} \nonumber \\
&=& \TNorm{\matE}\TNormS{\ta^{*T} \matY^{\textbf{sv}}\matU\matSig\matV^T} \nonumber \\
&=& \TNorm{\matE}\TNormS{\ta^{*T} \matY^{\textbf{sv}}\matX^{\textbf{sv}}}. \nonumber
\end{eqnarray*}
We can rewrite the above inequality as 
\begin{eqnarray}
&& \abs{\TNormS{\ta^{*T}\matY^{\textbf{sv}}\matX^{\textbf{sv}}\matR}- \TNormS{\ta^{*T}\matY^{\textbf{sv}}\matX^{\textbf{sv}}}} \nonumber \\
&\leq& \TNorm{\matE}\TNormS{\ta^{*T} \matY^{\textbf{sv}}\matX^{\textbf{sv}}} \nonumber \\
\nonumber &\leq& \frac{\TNorm{\matE}}{1-\TNorm{\matE}} \TNormS{\ta^{*T} \matY^{\textbf{sv}}\matX^{\textbf{sv}} \matR}.\nonumber
\end{eqnarray}
Combining with eqn.~\eqref{eqn:thm_eqn5}, we get
\begin{equation}\label{eqn:thm_eqn9}
Z_{opt} \geq  \tilde{Z}_{opt}-\frac{1}{2}\left(\frac{\TNorm{\matE}}{1-\TNorm{\matE}}\right)\TNormS{\ta^{*T} \matY^{\textbf{sv}}\matX^{\textbf{sv}}\matR}.
\end{equation}
Now recall that $\w^{*T} = \a^{*T}\matY^{\textbf{sv}}\matX^{\textbf{sv}}$, $\tw^{*T} = \ta^{*T}\matY^{\textbf{sv}}\matX^{\textbf{sv}}\matR$, $\TNormS{\w^*} = \sum_{i=1}^p \alpha_i^*$, and $\TNormS{\tw^*} = \sum_{i=1}^p \tilde\alpha_i^*$. Then, the optimal solutions $Z_{opt}$ and $\tilde{Z}_{opt}$ can be expressed as follows:
\begin{equation}\label{eqn:thm_eqn10}
Z_{opt} = \TNormS{\w^*} -\frac{1}{2}\TNormS{\w^*} = \frac{1}{2}\TNormS{\w^*}. 
\end{equation}\vskip -0.5cm
\begin{equation}\label{eqn:thm_eqn11}
\tilde{Z}_{opt} = \TNormS{\tw^*} -\frac{1}{2}\TNormS{\tw^*} = \frac{1}{2}\TNormS{\tw^*}.
\end{equation}
Combining eqns.~(\ref{eqn:thm_eqn9}), ~(\ref{eqn:thm_eqn10}) and ~(\ref{eqn:thm_eqn11}), we get
$\TNormS{\w^*} \geq \TNormS{\tw^*} - \left( \frac{\TNorm{\matE}}{1 - \TNorm{\matE}} \right)\TNormS{\tw^*}= \left(1-\frac{\TNorm{\matE}}{1 - \TNorm{\matE}} \right)\TNormS{\tw^*}.$
Let $\gamma^*=\norm{\w^*}_2^{-1}$ be the geometric margin of the problem of eqn.~(\ref{eqn:svm_sv}) and let $\tilde{\gamma}^*=\norm{\tw^*}_2^{-1}$ be the geometric margin of the problem of eqn.~(\ref{eqn:thm_eqn2}). Then, the above equation implies:
$\gamma^{*2} \leq \left(1-\frac{\TNorm{\matE}}{1 - \TNorm{\matE}} \right)^{-1}\tilde{\gamma}^{*2}$ 
$\Rightarrow \tilde{\gamma}^{*2} \geq  \left(1-\frac{\TNorm{\matE}}{1 - \TNorm{\matE}} \right)\gamma^{*2}.$
The result follows because \math{\norm{\matE}_2\le\epsilon/2} by our choice
of \math{r}, and so \math{\norm{\matE}_2/(1-\norm{\matE}_2)\le\epsilon}.
\end{proof}
\noindent
We now state a similar theorem for leverage-score sampling.

\begin{theorem}\label{thm:main_technical3}
Given \math{\epsilon\in(0,1)}, run supervised Leverage-score sampling based feature selection on
\math{\matX^{\textbf{sv}}} with \math{r_1=\tilde{O}(p/\epsilon^2)},
to obtain the feature sampling and rescaling matrix \math{\matR}.
Let \math{\gamma^*} and \math{\tilde\gamma^*} be the margins obtained by
solving the SVM dual \r{eqn:svm1A} with 
\math{(\matX^{\textbf{sv}},\matY^{\textbf{sv}})} and 
\math{(\matX^{\textbf{sv}} \matR,\matY^{\textbf{sv}})}
respectively. Then with probability at least 0.99,
$\tilde\gamma^{*2}\ge \left(1-\epsilon\right) \gamma^{*2}.$
\end{theorem}

\subsection{Geometry is preserved by Unsupervised Feature Selection}
In the unsupervised setting, the next theorem
says that with a number of features proportional to the 
rank of the training data (which is at most the number of data points), we
 preserve \math{B^2/\gamma^2}, 
thus ensuring comparable generalization error bounds
(\math{B} is the radius of the minimum enclosing ball).
\begin{theorem}
\label{thm:genthm}
Given \math{\epsilon\in(0,1)}, run unsupervised BSS-feature selection on
the full data
\math{\matX^{\textbf{tr}}} with \math{r_2=O\left(\rho/\epsilon^2 \right)},
where \math{\rho=\rank(\matX^{\textbf{tr}})},
to obtain the feature sampling and rescaling matrix
\math{\matR}.
Let \math{\gamma^*} and \math{\tilde\gamma^*} be the margins obtained by
solving the SVM dual \r{eqn:svm1A} with 
\math{(\matX^{\textbf{tr}},\matY^{\textbf{tr}})} and 
\math{(\matX^{\textbf{tr}} \matR,\matY^{\textbf{tr}})}
respectively; and, 
let $B$ and $\tilde{B}$ be the radii for the data matrices
 $\matX^{\textbf{tr}}$ and  $\matX^{\textbf{tr}}\matR$ respectively. Then,
$$\frac{\tilde{B}^2}{\tilde{\gamma}^{*2}} \leq \frac{\left(1+\epsilon\right)}{\left(1-\epsilon\right)}\frac{B^2}{\gamma^{*2}}
=
(1+O(\epsilon))\frac{B^2}{\gamma^{*2}}.$$	
\end{theorem}
\begin{proof}(sketch)
The proof has two parts. First, as in  Theorem~\ref{thm:main_technical}
we prove that $\tilde\gamma^{*2}\ge \left(1-\epsilon\right)\cdot \gamma^{*2}.$
This proof is almost identical to the proof of Theorem~\ref{thm:main_technical}
(replacing \math{(\matX^{\textbf{sv}},\matY^{\textbf{sv}})} with
\math{(\matX^{\textbf{tr}},\matY^{\textbf{tr}})}), and so we omit it.
Second, we prove that
\math{\tilde B^2\le(1+\epsilon)B^2}. We give the result (with proof)
as Theorem
\ref{thm:second_result-app} .
The theorem follows by combining these two results.
\end{proof}

\noindent
We now state a similar theorem for leverage-score sampling.

\begin{theorem}
\label{thm:genthm2}
Given \math{\epsilon\in(0,1)}, run unsupervised Leverage-score feature selection on
the full data
\math{\matX^{\textbf{tr}}} with \math{r_2=\tilde{O}\left(\rho/\epsilon^2 \right)},
where \math{\rho=\rank(\matX^{\textbf{tr}})},
to obtain the feature sampling and rescaling matrix
\math{\matR}.
Let \math{\gamma^*} and \math{\tilde\gamma^*} be the margins obtained by
solving the SVM dual \r{eqn:svm1A} with 
\math{(\matX^{\textbf{tr}},\matY^{\textbf{tr}})} and 
\math{(\matX^{\textbf{tr}} \matR,\matY^{\textbf{tr}})}
respectively; and, 
let $B$ and $\tilde{B}$ be the radii for the data matrices
 $\matX^{\textbf{tr}}$ and  $\matX^{\textbf{tr}}\matR$ respectively. Then with probability at least 0.99,
$$\frac{\tilde{B}^2}{\tilde{\gamma}^{*2}} \leq \frac{\left(1+\epsilon\right)}{\left(1-\epsilon\right)}\frac{B^2}{\gamma^{*2}}
=
(1+O(\epsilon))\frac{B^2}{\gamma^{*2}}.$$	
\end{theorem}

\subsection{Proof That the Data Radius is preserved by 
Unsupervised BSS-Feature Selection.}
 
\begin{theorem}\label{thm:second_result-app}
Let $r_2=O\left(n/\epsilon^2\right)$, where $\epsilon>0$ is an accuracy parameter, $n$ is the number of training points and $r_2$ is the number of features selected. Let $B$ be the radius of the minimum ball enclosing all points in the full-dimensional space, and let $\tilde{B}$ be the radius of the ball enclosing all points in the sampled subspace obtained by using BSS in an unsupervised manner. For $\matR$ as in Lemma~\ref{lem:ssp-app},
$\tilde{B}^2 \leq (1+\epsilon) B^2.$
\end{theorem}

\begin{proof}
We consider the matrix $\matX_B \in \mathbb{R}^{(n+1)\times d}$ whose first $n$ rows are the rows of $\matX^{\textbf{tr}}$ and whose last row is the vector $\x_B^T$; here $\x_B$ denotes the center of the minimum radius ball enclosing all $n$ points. Then, the SVD of $\matX_B$ is equal to $\matX_B = \matU_B \matSig_B \matV_B^T$, where $\matU_B \in \mathbb{R}^{(n+1) \times \rho_B}$, $\matSig_B \in \mathbb{R}^{\rho_B \times \rho_B}$, and $\matV \in \mathbb{R}^{d \times \rho_B}$. Here $\rho_B$ is the rank of the matrix $\matX_B$ and clearly $\rho_B \leq \rho + 1$. (Recall that $\rho$ is the rank of the matrix $\matX^{\textbf{tr}}$.) Let $B$ be the radius of the minimal radius ball enclosing all $n$ points in the original space. Then, for any $i=1,\ldots,n$,
\vskip -0.6cm
\begin{equation}\label{eqn:pd4}
B^2 \geq \TNormS{\x_i-\x_B} = \TNormS{\left(\e_i-\e_{n+1}\right)^T\matX_B}.
\end{equation}
%
Now consider the matrix $\matX_B\matR$ and notice that
\begin{eqnarray*}
 &&\abs{\TNormS{\left(\e_i-\e_{n+1}\right)^T\matX_B} - \TNormS{\left(\e_i-\e_{n+1}\right)^T\matX_B\matR}}\\
 &=& \abs{\left(\e_i-\e_{n+1}\right)^T\left(\matX_B\matX_B^T-\matX_B\matR\matR^T\matX_B^T\right)\left(\e_i-\e_{n+1}\right)}\\
%
%
&=& \abs{\left(\e_i-\e_{n+1}\right)^T\matU_B\matSig_B\matE_B\matSig_B\matU_B^T\left(\e_i-\e_{n+1}\right)}\\
&\leq& \TNorm{\matE_B}\TNormS{\left(\e_i-\e_{n+1}\right)^T\matU_B\matSig_B}\\
%
%
&=& \TNorm{\matE_B}\TNormS{\left(\e_i-\e_{n+1}\right)^T\matX_B}.
\end{eqnarray*}
In the above, we let $\matE_B \in \mathbb{R}^{\rho_B \times \rho_B}$ be the matrix that satisfies $\matV_B^T\matV_B = \matV_B^T \matR\matR^T \matV_B + \matE_B$, and we also used $\matV_B^T\matV_B = \matI$. Now consider the ball whose center is the $(n+1)$-th row of the matrix $\matX_B\matR$ (essentially, the center of the minimal radius enclosing ball for the original points in the sampled space). Let $\tilde{i} = \arg \max_{i=1\ldots n} \TNormS{\left(\e_i-\e_{n+1}\right)^T\matX_B\matR}$; then, using the above bound and eqn.~(\ref{eqn:pd4}), we get
%
%
$\TNormS{\left(\e_{\tilde{i}}-\e_{n+1}\right)^T\matX_B\matR} \leq \left(1+\TNorm{\matE_B}\right)\TNormS{\left(\e_{\tilde{i}}-\e_{n+1}\right)^T\matX_B} 
 \leq \left(1+\TNorm{\matE_B}\right)B^2.$
%
%
Thus, there exists a ball centered at $\e_{n+1}^T\matX_B\matR$ (the projected center of the minimal radius ball in the original space) with radius at most $\sqrt{1+\TNorm{\matE_B}}B$ that encloses all the points in the sampled space. Recall that $\tilde{B}$ is defined as the radius of the minimal radius ball that encloses all points in sampled subspace; clearly,
$\tilde{B}^2 \leq \left(1+\TNorm{\matE_B}\right)B^2.$
We can now use Lemma~\ref{lem:ssp-app} on $\matV_B$ to conclude that (using $\rho_B \leq \rho + 1$) $\TNorm{\matE_B} \leq \epsilon.$ 
\end{proof}

\section{Experiments}
We compared BSS and leverage-score sampling with RFE (\cite{Guyon02}), LPSVM (\cite{LPSVM}), rank-revealing QR factorization (RRQR), random feature selection and full-data without feature selection on synthetic and real-world datasets. For the supervised case, we first run SVM on the training set, then run a feature selection method on the support-vector set and recalibrate the model using the support vector-set. For unsupervised feature selection, we perform feature selection on the training set. For LPSVM, we were not able to control the number of features and report the out-of-sample error using the features output by the algorithm. We \textbf{did not} extrapolate the values of out-of-sample error for LPSVM. We repeated random feature selection and leverage-score sampling five times. We performed ten-fold cross-validation and repeated it ten times. For medium-scale datasets like TechTC-300 we do not perform approximate BSS. For large-scale datasets like Reuters-CCAT (\cite{DL04b}) we use the approximate BSS method as described in Section~\ref{subsec:approx_bss}. We used LIBSVM (\cite{Chang11}) as our SVM solver for medium-scale datasets and LIBLINEAR (\cite{Fan08}) for large-scale datasets. We do not report running times in our experiments, since feature selection is an offline-task. We implemented all our algorithms in MATLAB R2013b on an Intel i-7 processor with 16GB RAM. BSS and leverage-score sampling are better than LPSVM and RRQR and comparable to RFE on 49 TechTC-300 datasets. 

\subsection{Other Feature Selection Methods}
In this section, we describe other feature-selection methods with which we compare BSS and Leverage-score sampling.

\noindent\textbf{Rank-Revealing QR Factorization (RRQR):}
Within the numerical linear algebra community, subset selection algorithms use the so-called Rank Revealing QR (RRQR) factorization.
Let $\matA$ be a $n\times d$ matrix with $\left(n<d\right)$ and an integer $k \left(k<d\right)$ and assume partial QR factorizations of the form $$\matA\matP = \matQ \begin{pmatrix} \matR_{11} & \matR_{12} \\ \mathbf{0} & \matR_{22} \end{pmatrix},$$ 
where $\matQ \in \mathbb{R}^{n\times n}$ is an orthogonal matrix, $\matP  \in \mathbb{R}^{d\times d}$ is a permutation matrix, $\matR_{11} \in \mathbb{R}^{k\times k}, \matR_{12} \in \mathbb{R}^{k\times (d-k)},  \matR_{22} \in \mathbb{R}^{(d-k)\times (d-k)}$ The above factorization is called a RRQR factorization if $\sigma_{min}\left( \matR_{11}\right) \geq \sigma_k\left(\matA\right)/p(k,d)$, $\sigma_{max}\left( \matR_{22}\right) \leq \sigma_{min}(\matA) p(k,d),$ where $p(k,d)$ is a function bounded by a low-degree polynomial in $k$ and $d$. The important columns are given by $\matA_1= \matQ \begin{pmatrix} \matR_{11} \\ \mathbf{0} \end{pmatrix}$ and $\sigma_i \left(\matA_1\right) = \sigma_i\left(\matR_{11}\right)$ with $1\leq i \leq k.$ We perform feature selection using RRQR by picking the important columns which preserve the rank of the matrix.

\noindent \textbf{Random Feature Selection:}
We select features uniformly at random without replacement which serves as a baseline method. To get around the randomness, we repeat the sampling process five times.

\noindent\textbf{Recursive Feature Elimination:} 
Recursive Feature Elimination (RFE), \cite{Guyon02} tries to find the best subset of features which leads to the largest margin of class separation using SVM. At each iteration, the algorithm greedily removes the feature that decreases the margin the least, until the required number of features remain. At each step, it computes the weight vector and removes the feature with smallest weight. RFE is computationally expensive for high-dimensional datasets. Therefore, at each iteration, multiple features are removed to avoid the computational bottleneck.
 
\noindent \textbf{LPSVM:} 
The feature selection problem for SVM can be formulated in the form of a linear program. LPSVM \cite{LPSVM} uses a fast Newton method to solve this problem and obtains a sparse solution of the weight vector, which is used to select the features.

\subsection{BSS Implementation Issues}
At every iteration, there can be multiple columns which satisfy the condition $\mathcal{U}\left(\v_i,\delta_U,\matA_{\tau},U_{\tau} \right) \leq \mathcal{L}\left(\v_i,\delta_L, \matA_{\tau},L_{\tau} \right).$ \cite{BSS09} suggest picking any column which satisfies this constraint. Selecting a column naively leaves out important features required for classification. Therefore, we choose the column $\v_i$ which has not been selected in previous iterations and whose Euclidean-norm is highest among the candidate set. Columns with zero Euclidean norm never get selected by the algorithm.\\
In our implementation, we do not use the
data center as one of the inputs (since computing the center involves
solving a quadratic program).
\begin{table*}[!ht]
\caption{\small Most frequently selected features using the synthetic dataset.}
\label{tab:synth}
\begin{center}
\begin{small}
\begin{tabular}{|c||c|c||c|c|c|}
\hline
&\multicolumn{2}{c||}{$r_1=30$} &\multicolumn{2}{c|}{$r_1=40$} \\
\hline
&$k=40$ & $k=50$   &$k=40$ & $k=50$ \\
\hline
BSS	& \textbf{40}, 39, 34, 36, 37 & \textbf{50}, 49, 48, 47, 44  & \textbf{40}, 39, 34, 37, 36 &\textbf{50}, 49, 48, 47, 44 \\
\hline
Lvg    &\textbf{40}, 39, 37, 36, 34 &\textbf{50}, 49, 48, 47, 46  &\textbf{40}, 39, 37, 31, 32 &\textbf{50}, 49, 48, 31, 47 \\
\hline
RFE   & \textbf{40}, 39, 38, 37, 36 & \textbf{50}, 49, 48, 47, 46 &\textbf{40}, 39, 38, 37, 36 & \textbf{50}, 49, 48, 47, 46 \\
\hline
LPSVM &\textbf{40}, 39, 38, 37, 34 &\textbf{50}, 49, 48, 43, 40  &\textbf{40}, 39, 38, 37, 34 &\textbf{50}, 49, 48, 43, 40 \\
\hline
RRQR   &\textbf{40}, 30, 29, 28, 27  &  \textbf{50}, 30, 29, 28, 27  &  \textbf{40}, 39, 38, 37, 36  &  \textbf{50}, 40, 39, 38, 37  \\
\hline
\end{tabular}
\end{small}
\end{center}
\end{table*}

\subsection{Experiments on Supervised Feature Selection}
\noindent \textbf{Synthetic Data:}
We generate synthetic data as described in \cite{Bhat04}, where we control the number of relevant features in the dataset. The dataset has $n$ data-points and $d$ features. The class label $y_i$ of each data-point was randomly chosen to be 1 or -1 with equal probability. The first $k$ features of each data-point $\x_i$ are the relevant features and are drawn from $y_i \mathcal{N}\left(-j,1\right)$ distribution, where $\mathcal{N}\left(\mu,\sigma^2\right)$ is a random normal distribution with mean $\mu$ and variance $\sigma^2$ and $j$ varies from 1 to $k$. The remaining $(d-k)$ features are chosen from a $\mathcal{N}(0,1)$ distribution and are noisy features. By construction, among the first $k$ features, the $kth$ feature has the most discriminatory power, followed by $(k-1)th$ feature and so on. We set $n$ to 200 and $d$ to 1000. We set $k$ to 40 and 50 and ran two sets of experiments. We set the value of $r_1$, i.e. the number of features selected, to 30 and 40 for all experiments. We performed ten-fold cross-validation and repeated it ten times. We used LIBSVM with default settings and set $C=1$. We compared with the other methods. The mean out-of-sample error was 0 for all methods for both $k=40$ and $k=50$. Table~\ref{tab:synth} shows the set of five most frequently selected features by the different methods for one such synthetic dataset. The top features picked up by the different methods are the relevant features by construction and also have good discriminatory power. This shows that supervised BSS and leverage-score sampling are as good as any other method in terms of feature selection. We repeated our experiments on ten different synthetic datasets and each time, the five most frequently selected features were from the set of relevant features. Thus, by selecting only 3\% -4\% of all features, we show that we are able to obtain the most discriminatory features along with good out-of-sample error using BSS and leverage-score sampling.\\

\begin{figure}[!ht]
\begin{center}
\includegraphics[height = 65mm,width=\columnwidth,clip]{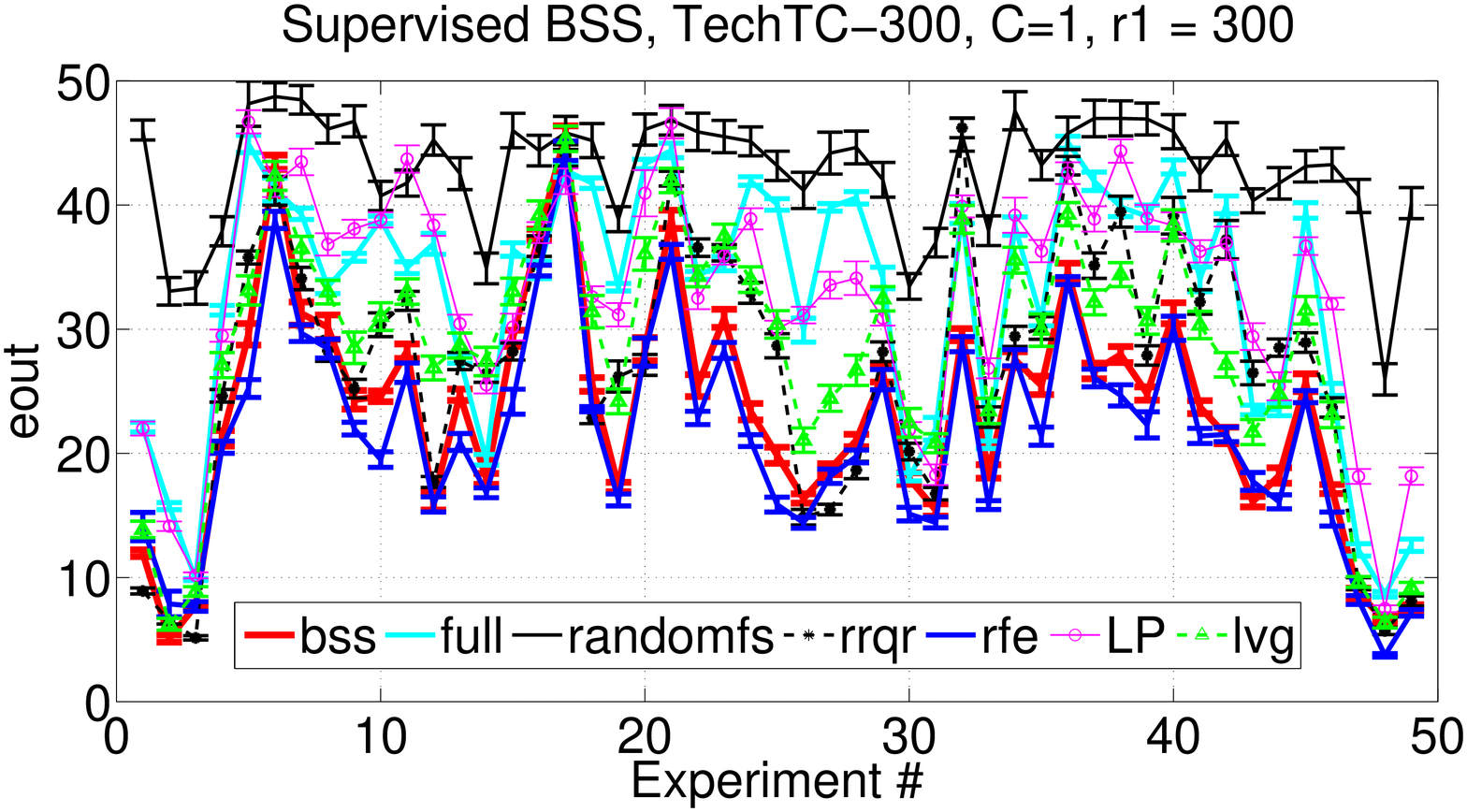}
\includegraphics[height = 65mm,width=\columnwidth,clip]{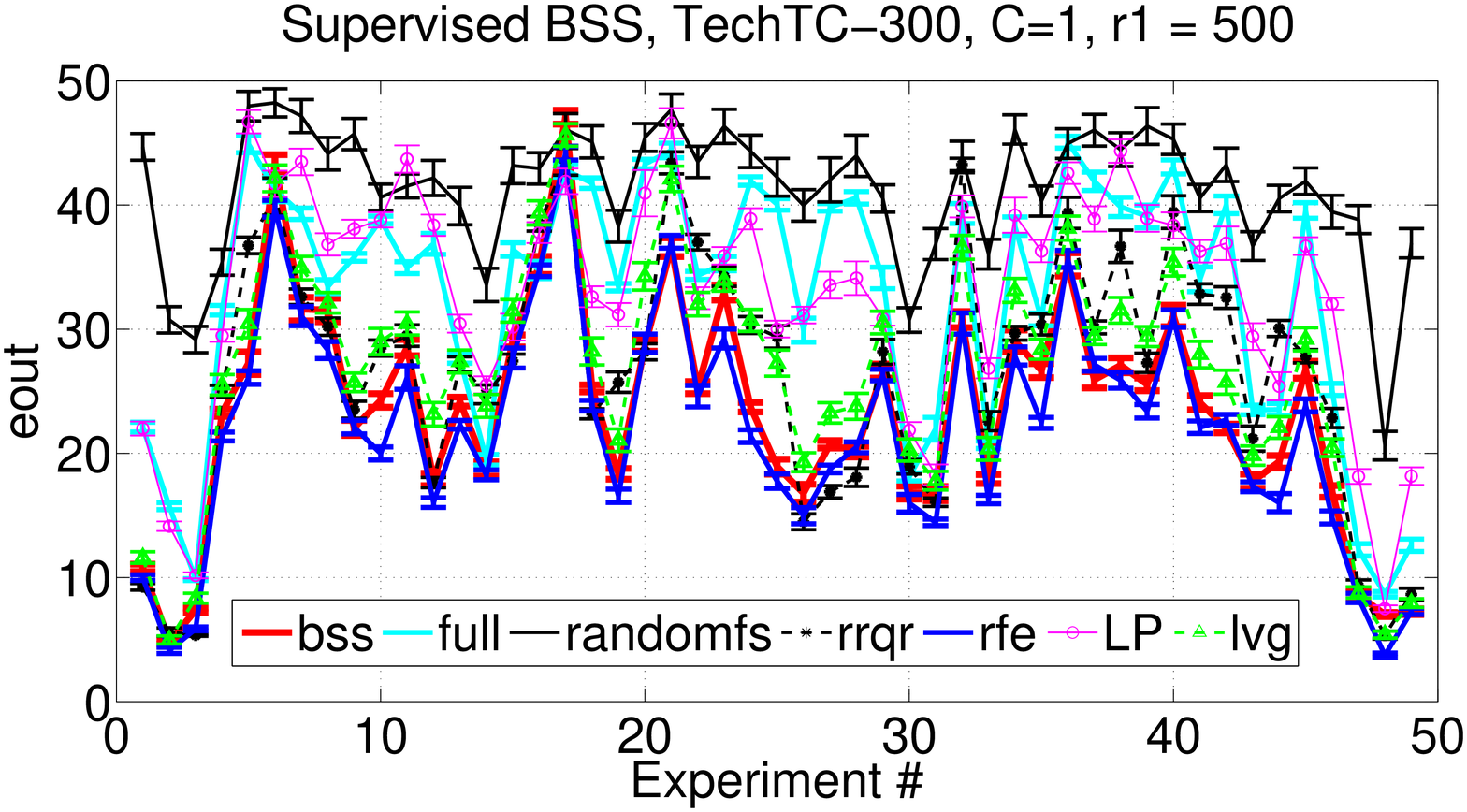}
Supervised Feature Selection \\[5pt]
\end{center}
\caption{\small Plots of out-of-sample error of Supervised BSS and leverage-score compared with other methods for 49 TechTC-300 documents averaged over ten ten-fold cross validation experiments. Vertical bars represent standard deviation.}
\label{fig:techtc_eout}
\end{figure}

\begin{figure}[!ht]
\begin{center}
\includegraphics[height = 65mm,width=\columnwidth,clip]{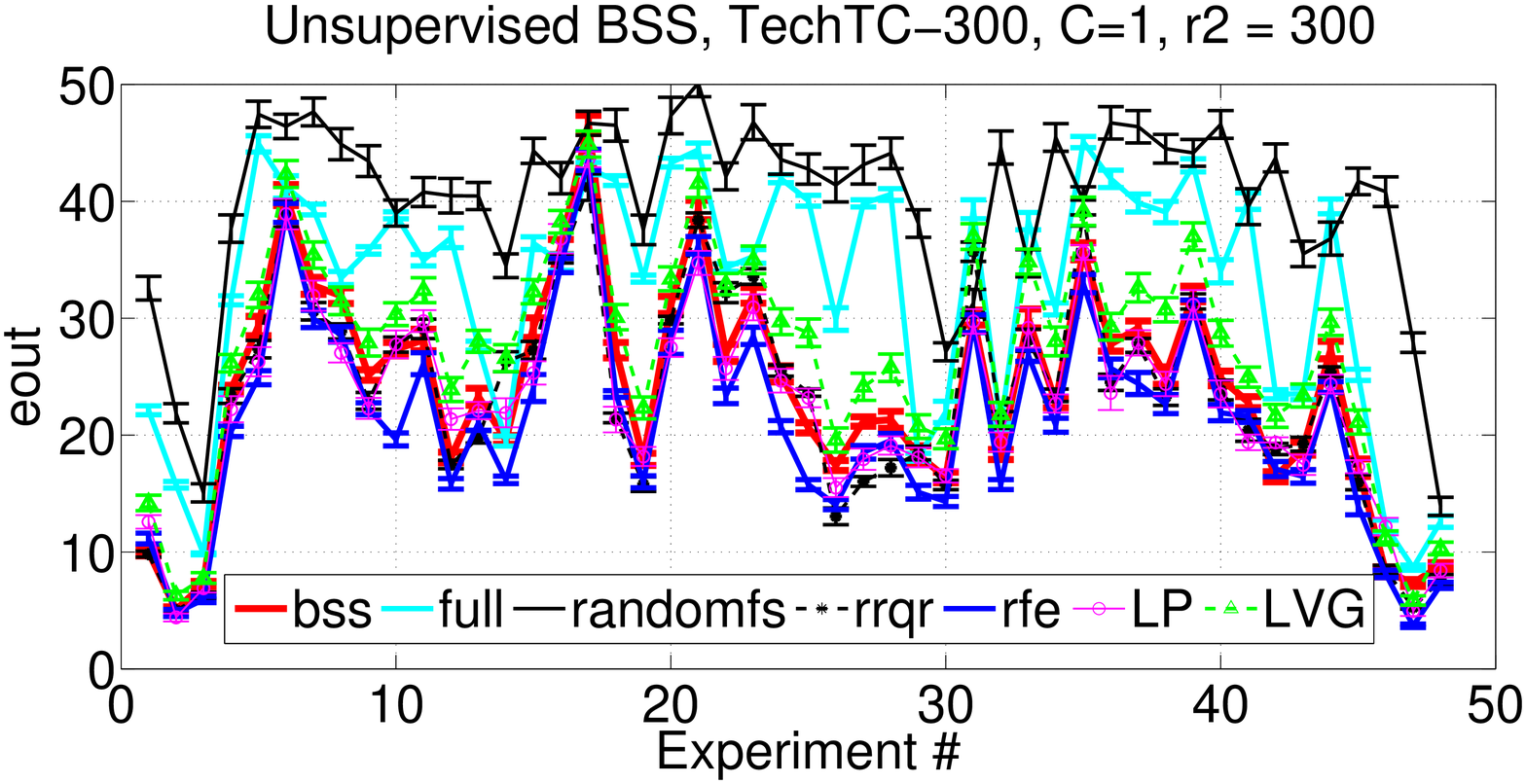}
\includegraphics[height = 65mm,width=\columnwidth,clip]{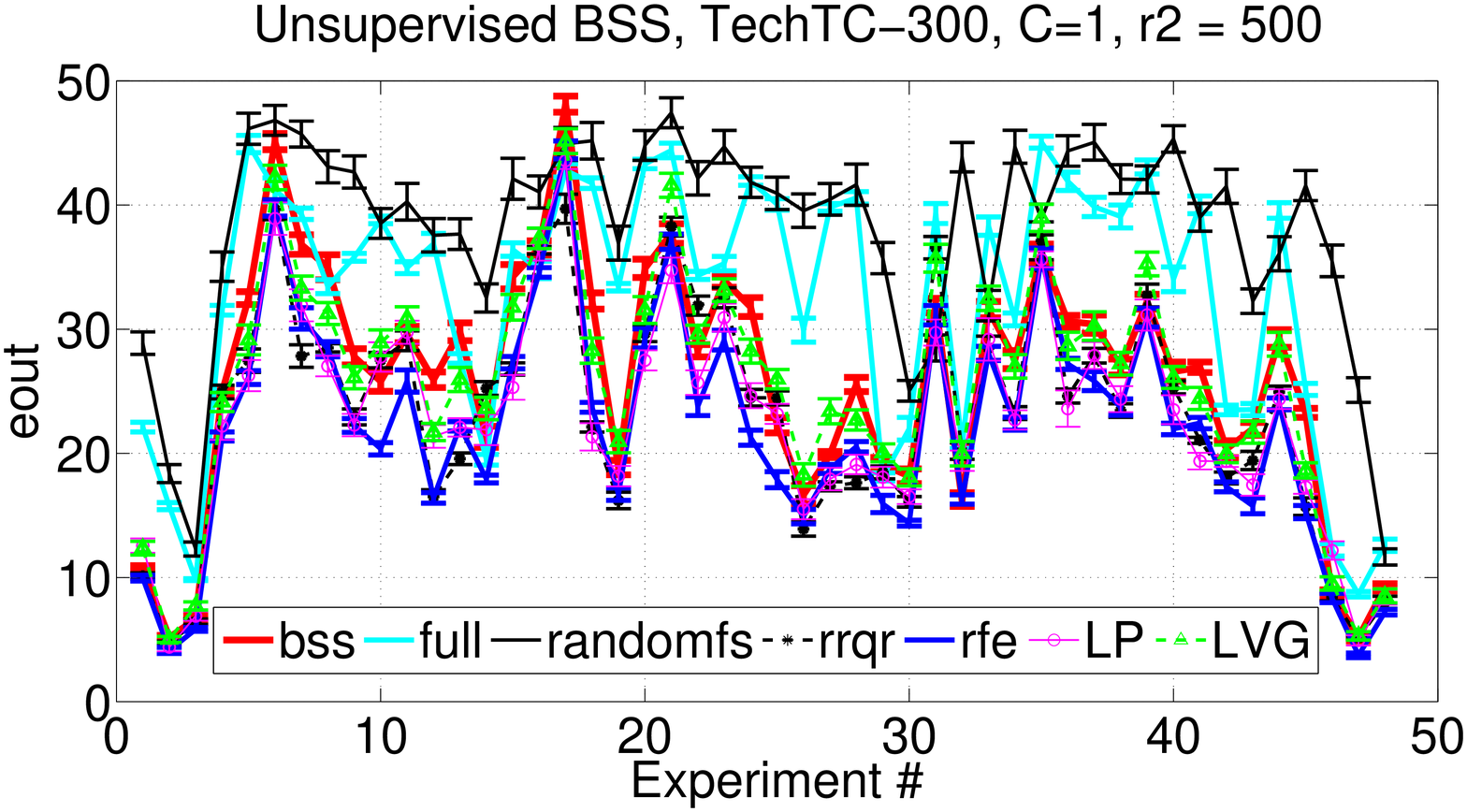}
Unsupervised Feature Selection
\end{center}
\caption{\small Plots of out-of-sample error of Unsupervised BSS and leverage-score compared with other methods for 49 TechTC-300 documents averaged over ten ten-fold cross validation experiments. Vertical bars represent standard deviation.}
\label{fig:techtc_eout_unsup}
\end{figure}

\begin{figure}[!htb]
\begin{center}
\includegraphics[height = 65mm,width=\columnwidth,clip]{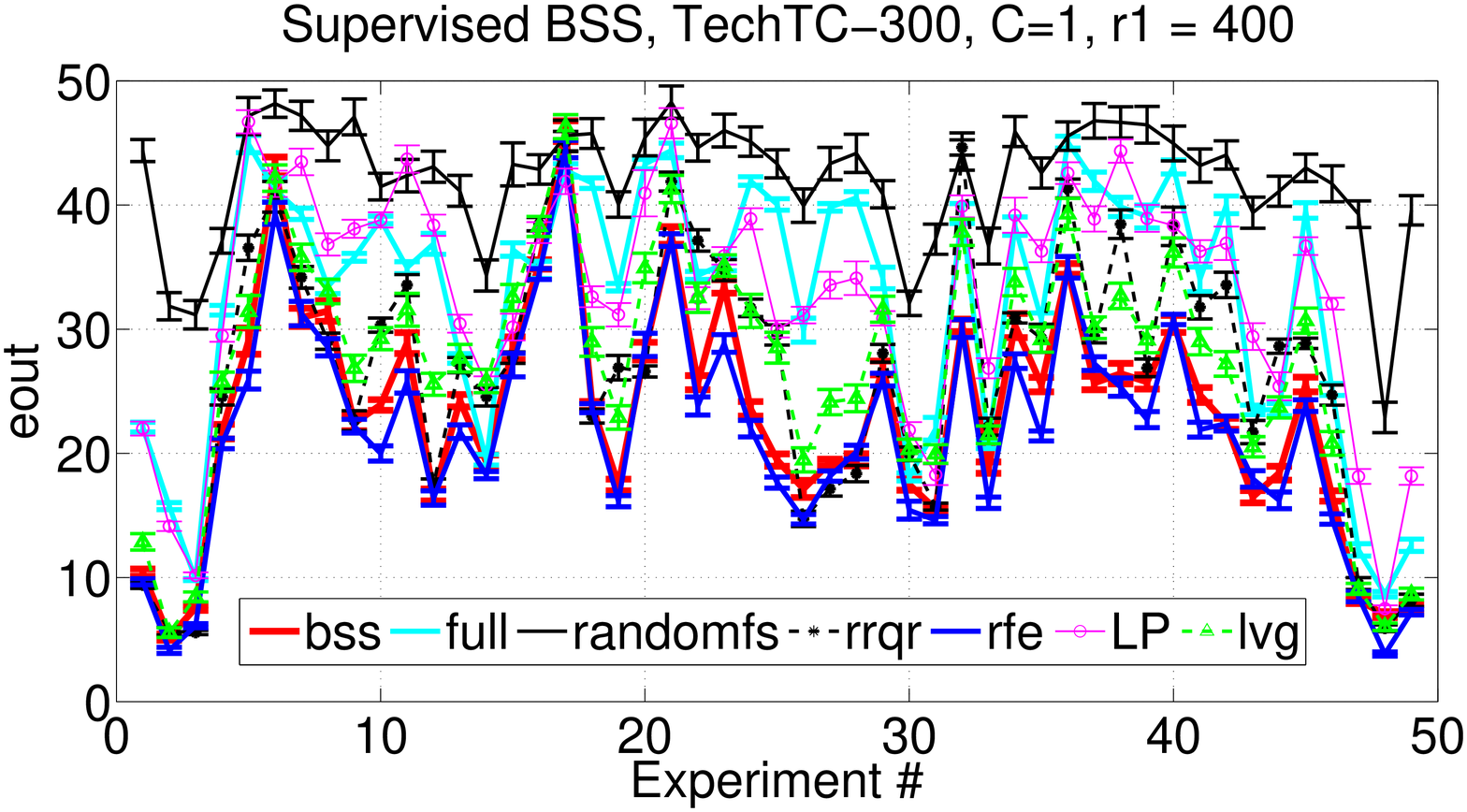}
Supervised Feature Selection \\[5pt]
\includegraphics[height = 65mm,width=\columnwidth,clip]{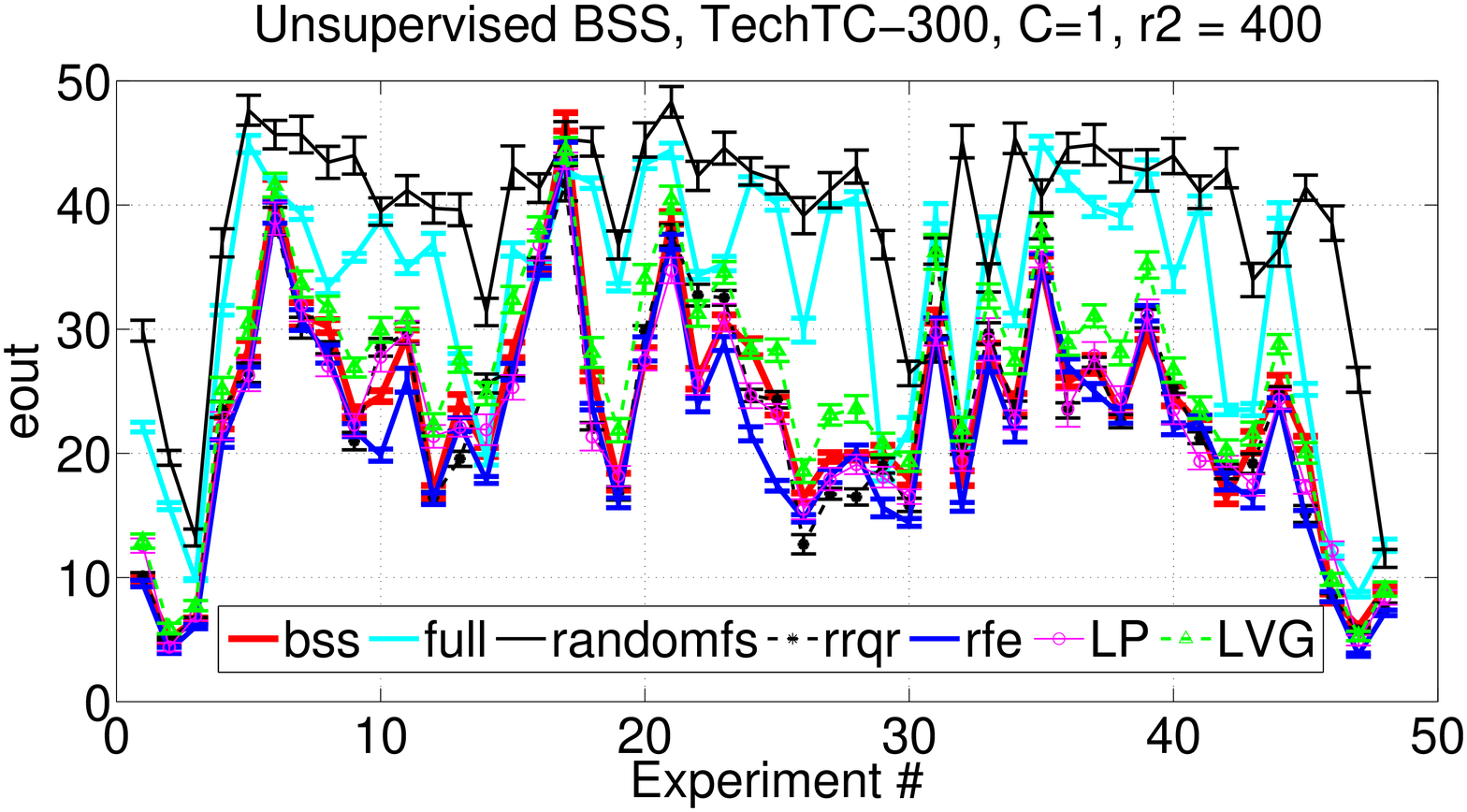}

Unsupervised Feature Selection
\end{center}
\caption{\small Plots of out-of-sample error of Supervised and Unsupervised BSS and leverage-score compared with other methods for 49 TechTC-300 documents averaged over ten ten-fold cross validation experiments. Vertical bars represent standard deviation.}
\label{fig:techtc_eout_400}
\end{figure}

\begin{table*}[!ht]
\centering
\caption{A subset of the TechTC matrices of our study}
\label{tab:filenames}
\begin{small}
\begin{tabular}{|l|l|l|}
\hline
  &\textbf{id1}  &\textbf{id2}\\ \hline
(i) &Arts: Music: Styles: Opera &Arts: Education: Language: Reading Instructions \\ \hline
(ii) &Arts: Music: Styles: Opera  &US Navy: Decommisioned Attack Submarines \\ \hline
(iii) &US: Michigan: Travel \& Tourism &Recreation:Sailing Clubs: UK \\ \hline
(iv)  &US: Michigan: Travel \& Tourism &Science: Chemistry: Analytical: Products \\ \hline
 (v) &US: Colorado: Localities: Boulder &Europe: Ireland: Dublin: Localities \\ \hline	
\end{tabular}
\end{small}
\end{table*}
%
\begin{table*}[!htb]
\begin{center}
\caption{Frequently occurring terms of the five TechTC-300 datasets of Table~\ref{tab:filenames} selected by supervised BSS and Leverage-score sampling.}
\label{tab:techtc_words}
\begin{small}
\begin{tabular}{|l|l|l|}
\hline
 &\textbf{BSS} &\textbf{Leverage-score Sampling} \\ \hline
(i) &reading, education, opera, frame & reading, opera, frame, spacer\\ \hline 
(ii) &submarine, hullnumber, opera, tickets &hullnumber, opera, music, tickets \\ \hline 
(iii) &michigan, vacation, yacht, sailing &sailing, yacht, michigan, vacation \\ \hline 
(iv) &chemical, michigan, environmental, asbestos &travel, vacation, michigan, services, environmental \\ \hline
(v) &ireland, dublin, swords, boulder, colorado &ireland, boulder, swords, school, grade \\ \hline
\end{tabular}
\end{small}
\end{center}
\end{table*}
\begin{table*}[!ht]
\caption{\small Results of Approximate BSS. CCAT (train / test): (23149 / 781265), d=47236. Mean and standard deviation (in parenthesis) of out-of-sample error. Eout of full-data is 8.66 $\pm$ 0.54.}
\label{tab:eout_rcv_svset}
\begin{center}
\begin{small}
\begin{tabular}{|c|c|c|c|c|c|c|}
\hline
Eout &$r_1$ &BSS ($t=128$) &BSS ($t=256$) &RRQR &RFE &LPSVM \\
\hline
CCAT &1024  &10.53 (0.59)    &10.35 (0.64)      &9.97 (0.62)          &8.92 (0.57)      &9.97 (0.55) \\
\hline
CCAT &2048  &11.13 (0.66)  &10.63 (0.62)   &10.04 (0.66)  &8.56 (0.54)  &9.97 (0.55) \\
\hline
\end{tabular}
\end{small}
\end{center}
\end{table*}

\noindent \textbf{TechTC-300:}
For our first real dataset, we use 49 datasets of TechTC-300 (\cite{David04}) which contain binary classification tasks. Each data matrix consists of 150-280 documents (the rows of the data matrix), and each document is described with respect to 10,000-50,000 words (features are columns of the matrix). We removed all words with at most four letters from the datasets. We set the parameter $C=1$ in LIBSVM and used default settings. We tried different values of $C$ for the full-dataset and the out-of-sample error averaged over 49 TechTC-300 documents did not change much, so we report the results of $C=1$. 
We set the number of features to 300, 400 and 500.  Figs~\ref{fig:techtc_eout} and ~\ref{fig:techtc_eout_400} show the out-of-sample error for the 49 datasets for $r1=300$, 400 and 500. For the supervised feature selection, BSS is comparable to RFE and leverage-score sampling and better than RRQR, LPSVM, full-data and uniform sampling in terms of out-of-sample error. For LPSVM, the number of selected features averaged over 49 datasets was greater than 500, but it performed worse than BSS and leverage-score sampling.
Leverage-score sampling is comparable to BSS and better than RRQR, LPSVM, full-data and uniform sampling and slightly worse than RFE.\\
\noindent We list the most frequently occurring words selected by supervised BSS and leverage-score for the $r_1=300$ case for five TechTC-300 datasets  over 100 training sets. Table~\ref{tab:filenames} shows the names of the five TechTC-300 document-term matrices. The words shown in Table ~\ref{tab:techtc_words} were selected in all cross-validation experiments for these five datasets. The words are closely related to the categories to which the documents belong, which shows that BSS and Leverage-score sampling select important features from the support-vector matrix. For example, for the document-pair $(ii)$, where the documents belong to the category of ``Arts:Music:Styles:Opera" and ``US:Navy: Decommisioned Attack Submarines", the BSS algorithm selects submarine, hullnumber, opera, tickets and Leverage-score sampling selects hullnumber, opera, music, tickets which are closely related to the two classes. Thus, we see that using only 2\%-4\% of all features we are able to obtain good out-of-sample error.
%
\subsection{Experiments on Unsupervised Feature Selection}
For the unsupervised feature selection case, we performed experiments on the same 49 TechTC-300 datasets and set $r_2$ to 300, 400 and 500.We include the results for $r_2=300$ and $r_2=500$ in Figs~\ref{fig:techtc_eout_unsup} and ~\ref{fig:techtc_eout_400}. For LPSVM, the number of selected features averaged over 49 datasets was close to 300. In the unsupervised case, BSS and leverage-score sampling are comparable to each other and also comparable to the other methods RRQR, LPSVM and RFE. These methods are better than random feature selection and full-data without feature selection. This shows that unsupervised BSS and leverage-score sampling are competitive feature selection algorithms.\\
\noindent Supervised feature selection is comparable to unsupervised feature selection for BSS, Leverage-score sampling and RFE, while unsupervised RRQR and LPSVM are better than their supervised versions. Running BSS (or leverage-score sampling) on the support-vector set is equivalent to running BSS (or leverage-score sampling) on the training data. However, RRQR and LPSVM are primarily used as unsupervised feature selection techniques and so they perform well in that setting. RFE is a heuristic based on SVM and running RFE on the support-vectors is equivalent to running RFE on the training data.

\subsection{Approximate BSS}
\label{subsec:approx_bss}
We describe a heuristic to make supervised BSS scalable to large-scale datasets. For datasets with large number of support vectors, we premultiply the support vector matrix $\matX$ with a random gaussian matrix $\matG \in \mathbb{R}^{t\times p}$ to obtain $\hat{\matX}=\matG\matX$ and then use BSS to select features from the right singular vectors of $\hat{\matX}$. The right singular vectors of $\hat\matX$ closely approximates the right singular vectors of $\matX$. Hence the columns selected from $\hat{\matX}$ will be approximately same as the columns selected from $\matX$. We include the algorithm as Algorithm~\ref{alg:approx_bss}.
We performed experiments on a subset of Reuters Corpus dataset, namely reuters-CCAT, which contains binary classification task. 
We used the L2-regularized L2-loss SVM formulation in the dual form in LIBLINEAR and set the value of $C$ to 10. 
We experimented with different values of $C$ on the full-dataset, and since there was small change in classificaton accuracy among the different values of $C$, we chose $C=10$ for our experiments. We pre-multiplied the support vector matrix with a random gaussian matrix of size $t\times p$, where $p$ is the number of support vectors and $t$ was set to 128 and 256. We repeated our experiments five times using five different random gaussian matrices to get around the randomness. We set the value of $r_1$ in BSS to 1024 and 2048. LPSVM selects 1898 features for CCAT. Table~\ref{tab:eout_rcv_svset} shows the results of our experiments. We observe that the out-of-sample error using approx-BSS is close to that of RRQR and comparable to RFE, LPSVM and full-data. The out-of-sample error of approx-BSS decreases with an increase in the value of $t$. This shows that we get a good approximation of the right singular vectors of the support vector matrix with an increase in number of projections. 

\begin{algorithm}[!hbt]
\begin{framed}

\textbf{Input:} Support vector matrix $\matX \in \mathbb{R}^{p\times d}$, $t, r.$ \\
\textbf{Output:} Matrices $\matS \in \mathbb{R}^{d\times r}, \matD \in \mathbb{R}^{r\times r}$.

\begin{enumerate}

\item Generate a random Gaussian matrix, $\matG \in \mathbb{R}^{t\times p}$.
\item Compute $\hat{\matX}=\matG\matX$.
\item Compute right singular vectors $\matV$ of $\hat\matX$ using SVD.
\item Run Algorithm~\ref{alg:alg_ssp} using $\matV$ and $r$ as inputs and get matrices $\matS$ and $\matD$ as outputs.

\item Return $\matS$ and $\matD.$
\end{enumerate}
\end{framed}
\caption{Approximate BSS}
\label{alg:approx_bss}
\end{algorithm}

\section{Conclusions}
Our simple method of extending an unsupervised feature 
selection method into a supervised one for SVM not only has a 
provable guarantee, but also works well empirically:
BSS and leverage-score sampling are comparable and often better than prior state-of-the-art 
feature selection methods for SVM, and those methods 
don't come with guarantees. \\
Our supervised sparsification algorithms only preserve the
margin for the support vectors in the
feature space. We do not make any claims about the margin of the
full data in the feature space constructed from the support vectors.
This appears challenging and it would be interesting to 
see progress made in this direction: can one
choose \math{O(\# \text{support vectors})} features for the
full data set and obtain provable guarantees on the margin and data radius?
There have been recent advances in approximate leverage-scores for large-scale datasets. A possible future work in this direction would be to see if those algorithms indeed work well with SVMs. 

\section{Acknowledgements}
PD and SP are supported by NSF IIS-1447283 and IIS-1319280 respectively.

\begin{small}
\bibliographystyle{unsrt}
\bibliography{references}

\begin{thebibliography}{10}

\bibitem{Dasgup07}
A.~Dasgupta, P.~Drineas, B.~Harb, V.~Josifovski, and M.W. Mahoney.
\newblock Feature selection methods for text classification.
\newblock In {\em Proceedings of the 13th ACM SIGKDD International Conference
  on Knowledge Discovery and Data Mining}, pages 230--239, 2007.

\bibitem{Chris00}
N.~Cristianini and J.~Shawe-Taylor.
\newblock {\em Support Vector Machines and other kernel-based learning
  methods}.
\newblock Cambridge University Press, 2000.

\bibitem{Guyon02}
I.~Guyon, J.~Weston, S.~Barnhill, and V.~Vapnik.
\newblock Gene selection for cancer classification using support vector
  machines.
\newblock {\em Machine Learning}, 46(1-3):389--422, 2002.

\bibitem{LPSVM}
M.~Glenn Fung and O.L. Mangasarian.
\newblock A feature selection newton method for support vector machine
  classification.
\newblock {\em Comput. Optim. Appl.}, 28(2):185--202, 2004.

\bibitem{Vapnik71}
V.N. Vapnik.
\newblock Statistical {L}earning {T}heory.
\newblock {\em Theory of Probability and its Applications}, 16:264--280, 1998.

\bibitem{Vapnik98}
V.N. Vapnik and A.~Chervonenkis.
\newblock On the {U}niform {C}onvergence of {R}elative {F}requencies of
  {E}vents to their {P}robabilities.
\newblock {\em Theory of Probability and its Applications}, 16:264--280, 1971.

\bibitem{BSS09}
J.D. Batson, D.A. Spielman, and N.~Srivastava.
\newblock Twice-ramanujan sparsifiers.
\newblock In {\em Proceedings of the 41st annual ACM STOC}, pages 255--262,
  2009.

\bibitem{Rak03}
A.~Rakotomamonjy.
\newblock Variable selection using svm based criteria.
\newblock {\em JMLR}, 3:1357--1370, 2003.

\bibitem{Weston00}
J.~Weston, S.~Mukherjee, O.~Chapelle, M.~Pontil, T.~Poggio, and V.~Vapnik.
\newblock Feature selection for svms.
\newblock In {\em NIPS}, volume~12, pages 668--674, 2000.

\bibitem{Weston03}
J.~Weston, A.~Elisseeff, B.~Sch{\"o}lkopf, and M.~Tipping.
\newblock Use of the zero norm with linear models and kernel methods.
\newblock {\em JMLR}, 3:1439--1461, 2003.

\bibitem{Tan10}
M.~Tan, L.~Wang, and I.W. Tsang.
\newblock Learning sparse svm for feature selection on very high dimensional
  datasets.
\newblock In {\em Proceedings of the 27th International Conference on Machine
  Learning (ICML)}, pages 1047--1054, 2010.

\bibitem{Do09}
H.~Do, A.~Kalousis, and M.~Hilario.
\newblock Feature weighting using margin and radius based error bound
  optimization in svms.
\newblock In {\em European Conference on Machine Learning (ECML)}, pages
  315--329, 2009b.

\bibitem{Do13}
A.~Kalousis and H.T. Do.
\newblock Convex formulations of radius-margin based support vector machines.
\newblock In {\em Proceedings of the 30th International Conference on Machine
  Learning (ICML)}, pages 169--177, 2013.

\bibitem{Wang06}
L.~Wang, J.~Zhu, and H.~Zou.
\newblock The doubly regularized support vector machine.
\newblock {\em Statistica Sinica}, 16(2):589, 2006.

\bibitem{Wang08}
L.~Wang, J.~Zhu, and H.~Zou.
\newblock Hybrid huberized support vector machines for microarray
  classification and gene selection.
\newblock {\em Bioinformatics}, 24(3):412--419, 2008.

\bibitem{Ye11}
G.~Ye, Y.~Chen, and X.~Xie.
\newblock Efficient variable selection in support vector machines via the
  alternating direction method of multipliers.
\newblock In {\em AISTATS}, pages 832--840, 2011.

\bibitem{Gilad04}
R.~Gilad-Bachrach, A.~Navot, and N.~Tishby.
\newblock Margin based feature selection-theory and algorithms.
\newblock In {\em Proceedings of the twenty-first international conference on
  Machine learning (ICML)}, page~43, 2004.

\bibitem{Park}
C.~Park, K-R. Kim, R.~Myung, and J-Y. Koo.
\newblock Oracle properties of scad-penalized support vector machine.
\newblock {\em Journal of Statistical Planning and Inference},
  142(8):2257--2270, 2012.

\bibitem{Paulaistats}
S.~Paul, C.~Boutsidis, M.~Magdon-Ismail, and P.~Drineas.
\newblock Random projections for support vector machines.
\newblock In {\em Sixteenth International Conference on Artificial Intelligence
  and Statistics (AISTATS)}, pages 498--506. JMLR W\&CP 31, 2013.

\bibitem{Paultkdd}
S.~Paul, C.~Boutsidis, M.~Magdon-Ismail, and P.~Drineas.
\newblock Random projections for linear support vector machines.
\newblock {\em ACM Trans. Knowl. Discov. Data}, 8(4):22:1--22:25, 2014.

\bibitem{Bouts13}
C.~Boutsidis and M.~Magdon-Ismail.
\newblock Deterministic feature selection for $ k $-means clustering.
\newblock {\em IEEE Transactions on Information Theory}, 59(9):6099-- 6110,
  2013.

\bibitem{Bouts09}
C.~Boutsidis, M.~W. Mahoney, and P.~Drineas.
\newblock Unsupervised feature selection for the k-means clustering problem.
\newblock {\em Advances in Neural Information Processing Systems (NIPS)}, 2009.

\bibitem{Paul14}
S.~Paul and P.~Drineas.
\newblock Deterministic feature selection for regularized least squares
  classification.
\newblock In {\em Machine Learning and Knowledge Discovery in Databases
  (ECML-PKDD)}, volume 8725 of {\em LNCS}, pages 533--548, 2014.

\bibitem{DL04b}
D.~D. Lewis, Y.~Yang, T.~G. Rose, and F.~Li.
\newblock Rcv1: A new benchmark collection for text categorization research.
\newblock {\em JMLR}, pages 361--397, 2004.

\bibitem{Chang11}
C-C. Chang and C-J. Lin.
\newblock Libsvm: A library for support vector machines.
\newblock {\em ACM Transactions on Intelligent Systems and Technology},
  2:27:1--27:27, 2011.
\newblock Software available at \url{http://www.csie.ntu.edu.tw/~cjlin/libsvm}.

\bibitem{Fan08}
R.-E. Fan, K.-W. Chang, C.-J. Hsieh, X.-R. Wang, and C.-J. Lin.
\newblock Liblinear: A library for large linear classification.
\newblock {\em JMLR}, pages 1871 --1874, 2008.

\bibitem{Bhat04}
C.~Bhattacharyya.
\newblock Second order cone programming formulations for feature selection.
\newblock {\em JMLR}, 5:1417--1433, 2004.

\bibitem{David04}
D.~Davidov, E.~Gabrilovich, and S.~Markovitch.
\newblock Parameterized generation of labeled datasets for text categorization
  based on a hierarchical directory.
\newblock In {\em Proceedings of the 27th Annual International ACM SIGIR
  Conference}, pages 250--257, 2004.
\newblock \url{http://techtc.cs.technion.ac.il/techtc300/techtc300.html}.

\end{thebibliography}
\end{small}

\end{document}